\documentclass{llncs}
\usepackage{amsmath}
\usepackage{times}
\usepackage{indentfirst}
\usepackage{graphicx}

\newtheorem{faulse assertion}{faulse assertion}
\newtheorem{counter-example}{counter-example}

\begin{document}

\title{Condition for neighborhoods induced by a covering to be equal to the covering itself}         

\author{Hua Yao , William Zhu\thanks{Corresponding author.
E-mail: williamfengzhu@gmail.com(William Zhu)} }
\institute{Lab of Granular Computing,\\
Zhangzhou Normal University, Zhangzhou, China}



\date{\today}          
\maketitle

\begin{abstract}
It is a meaningful issue that under what condition neighborhoods induced by a covering are equal to the covering itself. A necessary and sufficient condition for this issue has been provided by some scholars. In this paper, through a counter-example, we firstly point out the necessary and sufficient condition is false. Second, we present a necessary and sufficient condition for this issue. Third, we concentrate on the inverse issue of computing neighborhoods by a covering, namely giving an arbitrary covering, whether or not there exists another covering such that the neighborhoods induced by it is just the former covering. We present a necessary and sufficient condition for this issue as well. In a word, through the study on the two fundamental issues induced by neighborhoods, we have gained a deeper understanding of the relationship between neighborhoods and the covering which induce the neighborhoods.
\end{abstract}

\section{Introduction}
Rough set theory, proposed by Pawlak~\cite{Pawlak82Rough,Pawlak91Rough}, is an extension of set theory for the
study of intelligent systems characterized by insufficient and incomplete information. In theory, rough sets have been connected with matroids~\cite{TangSheZhu12matroidal,WangZhuZhuMin12matroidalstructure}, lattices~\cite{Dai05Logic,EstajiHooshmandaslDavvaz12Roughappliedtolattice,Liu08Generalized,WangZhu11Quantitative}, hyperstructure theory~\cite{YamakKazanciDavvaz11Softhyperstructure},
topology~\cite{Kondo05OnTheStructure,LashinKozaeKhadraMedhat05Rough,Zhu07Topological}, fuzzy sets~\cite{KazanciYamakDavvaz08TheLower,WuLeungMi05OnCharacterizations}, and so on. Rough set theory is built on an equivalence relation, or to say, on a partition. But equivalence relation or partition is still restrictive for many applications. To address
this issue, several meaningful extensions to equivalence relation
have been proposed. Among them, Zakowski has used coverings
of a universe for establishing the covering based rough set theory~\cite{Zakowski83Approximations}. Many scholars have done deep researches on this theory~\cite{BonikowskiBryniarskiWybraniecSkardowska98Extensions,Bryniarski89ACalculus,ZhuWang03Reduction}, and some basic results have been presented.

Neighborhood is an important concept in covering based rough set theory. Many scholars have studied it from different perspectives. Lin augmented the relational database with neighborhood~\cite{Lin88Neighborhoodsystems}. Yao presented a framework for the formulation, interpretation, and comparison of neighborhood systems and rough set approximations~\cite{Yao98Relational}. By means of consistent function based on the concept of neighborhood, Wang et al.~\cite{WangChenSunHu12Communication} dealt with information systems through covering based rough sets. Furthermore, the concept of neighborhood itself has produced lots of meaningful issues as well, and it is one of them that under what condition neighborhoods induced by a covering are equal to the covering itself. In paper~\cite{WangChenSunHu12Communication}, Wang et al. provided a necessary and sufficient condition about this issue.

In this paper, through a counter-example, we firstly point out that the necessary and sufficient condition provided by Wang et al. is false. Second, we propose the concepts of repeat degree and core block, and then study some properties of them. Third, we propose the concept of invariable covering based on core block. And by means of invariable covering, we present a necessary and sufficient condition for neighborhoods induced by a covering to be equal to the covering itself. Fourth, we concentrate on the inverse issue of computing neighborhoods by a covering, namely giving an arbitrary covering, whether or not there exists another covering such that the neighborhoods induced by it is just the former covering. By means of a property of neighborhoods obtained by Liu et al.~\cite{LiuSai09AComparison} and us independently, we present a necessary and sufficient condition for covering to be a neighborhoods induced by another covering.

The remainder of this paper is organized as follows. In Section~\ref{S:Preliminaries}, we review the relevant concepts and point out that the necessary and sufficient condition provided by Wang et al. is false. In Section~\ref{S:Some new concepts and their properties}, we propose the concepts of repeat degree and core block, and then study some properties of them. In Section~\ref{S:Condition for neighborhoods induced by a covering to
be equal to the covering itself}, we present a necessary and sufficient condition for neighborhoods induced by a covering to be equal to the covering itself. In Section~\ref{S:Condition for covering to be a neighborhoods}, we present a necessary and sufficient condition for covering to be a neighborhoods induced by another covering. Section~\ref{S:Conclusions} presents conclusions.

\section{Preliminaries}
\label{S:Preliminaries}
The concepts of partition and covering are the basis of classical rough sets and covering based rough sets, respectively. And covering is the basis of the concept of neighborhood as well. So we introduce the two concepts at first.

\begin{definition}(Partition)
\label{D:Partition}
Let $U$ be a universe of discourse and $\mathbf{P}$ a family of subsets of $U$. If $\emptyset\notin\mathbf{P}$, and $\cup \mathbf{P}=U$, and for any $K,L\in\mathbf{P}$, $K\cap L=\emptyset$, then $\mathbf{P}$ is called a partition of $U$. Every element of $\mathbf{P}$ is called a partition block.
\end{definition}

In the following discussion, unless stated to the contrary, the universe of discourse $U$ is considered to be
finite and nonempty.

\begin{definition}(Covering)
\label{D:Covering}
Let $U$ be a universe and $\mathbf{C}$ a family of subsets of $U$. If $\emptyset\notin\mathbf{C}$, and $\cup \mathbf{C}=U$, then $\mathbf{C}$ is called a covering of $U$. Every element of $\mathbf{C}$ is called a covering block.
\end{definition}

It is clear that a partition of $U$ is certainly a covering of $U$, so the concept of covering is an extension of the concept of partition.
In the following, we introduce the concepts of neighborhood and neighborhoods, two main concepts which will be discussed in this paper.

\begin{definition}(Neighborhood~\cite{Lin88Neighborhoodsystems})
\label{D:Neighborhood}
Let $\mathbf{C}$ be a covering of $U$. For any $x\in U$, $N(x)=\cap\{K\in\mathbf{C}|x\in K\}$ is called the neighborhood of $x$.
\end{definition}

A relationship between two different neighborhoods is presented by the following proposition.

\begin{proposition}~\cite{WangChenSunHu12Communication}
\label{P:1}
Let $\mathbf{C}$ be a covering of $U$. For any $x,y\in U$, if $y\in N(x)$, then $N(y)\subseteq N(x)$. So if $y\in N(x)$ and $x\in N(y)$, then $N(x)=N(y)$.
\end{proposition}

After the concept of neighborhood has been given, we can introduce the concept of neighborhoods.

\begin{definition}~\cite{WangChenSunHu12Communication}
\label{D:2}
Let $\mathbf{C}$ be a covering of $U$. $Cov(\mathbf{C})=\{N(x)|x\in U\}$ is called the neighborhoods induced by $\mathbf{C}$.
\end{definition}

There is an important property of neighborhoods presented by the following proposition.

\begin{proposition}~\cite{WangChenSunHu12Communication}
\label{P:0}
For any $N(x)\in Cov(\mathcal{C})$, $N(x)$ is not a union of other blocks in $Cov(\mathcal{C})$.
\end{proposition}

By the definition of $Cov(\mathbf{C})$, we see that $Cov(\mathbf{C})$ is still a covering of universe $U$. In particular, if $\mathbf{C}$ is a partition, we have that $Cov(\mathbf{C})=\mathbf{C}$. In paper~\cite{WangChenSunHu12Communication}, Wang et al. said that $Cov(\mathbf{C})=\mathbf{C}$ if and only if $\mathbf{C}$ was a partition. The following counter-example indicates that the necessity of this proposition is false.

\begin{example}
\label{E:3}
Let $U=\{1,2,3\}$, $\mathbf{C}=\{K_{1},K_{2},K_{3}\}$, where $K_{1}=\{1\}$, $K_{2}=\{1,2\}$, $K_{3}=\{3\}$. We have that $N(1)=\{1\}=K_{1}$, $N(2)=\{1,2\}=K_{2}$, $N(3)=\{3\}=K_{3}$, thus $Cov(\mathbf{C})=\{N(1),N(2),N(3)\}=\{K_{1},K_{2},K_{3}\}=\mathbf{C}$. But $\mathbf{C}=\{K_{1},K_{2},K_{3}\}=\{\{1\},\{1,2\},\{3\}\}$ is not a partition.
\end{example}

In the following sections, we firstly propose some new concepts, and then study on their properties. By means of them, we present a necessary and
sufficient condition for neighborhoods induced by a covering to be equal to the covering itself.

\section{Repeat degree and core block}
\label{S:Some new concepts and their properties}
There is a difference between a partition and a covering of a same universe $U$. The difference is embodied in that for any $x\in U$, there exists only one partition block which include $x$ but there might exist more than one covering block which include $x$. Then it is necessary to concern with how many blocks including $x$ there are in a covering. Inspired by this, we propose the following concept.

\begin{definition}(Membership repeat degree)
\label{D:Membership repeat degree}
Let $\mathbf{C}$ be a covering of a universe $U$. We define a function $\partial_{\mathbf{C}}:U\rightarrow N^{+}$, $\partial_{\mathbf{C}}(x)=|\{K\in\mathbf{C}|x\in K\}|$, and call $\partial_{\mathbf{C}}(x)$ the membership repeat degree of $x$ with respect to covering $\mathbf{C}$. When the covering is clear, we omit the lowercase $\mathbf{C}$ for the function.
\end{definition}

That an element $x$ of $U$ has the membership repeat degree of $\partial(x)$ means that there are $\partial(x)$  blocks in covering  $\mathbf{C}$ which include element $x$.
To illustrate the above definition, let us see an example.

\begin{example}
\label{E:4}
Let $U=\{1,2,3\}$, $\mathbf{C}=\{K_{1},K_{2}\}$, where $K_{1}=\{1,2\}$, $K_{2}=\{2,3\}$. Then $\{K\in\mathbf{C}|1\in K\}=\{K_{1}\}$, $\{K\in\mathbf{C}|2\in K\}=\{K_{1},K_{2}\}$, $\{K\in\mathbf{C}|3\in K\}=\{K_{2}\}$, thus $\partial(1)=|\{K_{1}\}|=1$, $\partial(2)=|\{K_{1},K_{2}\}|=2$, $\partial(3)=|\{K_{2}\}|=1$.
\end{example}

In order to learn more about the neighborhoods, a special kind of covering, it is not enough using membership repeat degree of single element. We need research further that how many blocks including $x$ and $y$ simultaneously there are in a covering.

\begin{definition}(Common block repeat degree)
\label{D:Common block repeat degree}
Let $\mathbf{C}$ be a covering of a universe $U$. We define a function $\lambda_{\mathbf{C}}: U\times U\rightarrow N, \lambda_{\mathbf{C}}((x,y))=|\{K\in\mathbf{C}|\{x,y\}\subseteq K\}|$. We write $\lambda_{\mathbf{C}}((x,y))$ as $\lambda_{\mathbf{C}}(x,y)$ for short, and for any $x,y\in U$, we call $\lambda_{\mathbf{C}}(x,y)$ the common block repeat degree of binary group $(x,y)$ with respect to covering $\mathbf{C}$. When the covering is clear, we omit the lowercase $\mathbf{C}$ for the function.
\end{definition}

That a binary group $(x,y)$ of universe $U$ has the common block repeat degree of $\lambda(x,y)$ with respect to covering $\mathbf{C}$ means that there are $\lambda(x,y)$ blocks in covering $\mathbf{C}$ which include element $x$ and $y$ simultaneously.
To illustrate the above definition, let us see an example.

\begin{example}
\label{E:5}
Let $U=\{1,2,3,4\}$, $\mathbf{C}=\{K_{1},K_{2},K_{3}\}$, where $K_{1}=\{1,2\}$, $K_{2}=\{2,3,4\}$, $K_{3}=\{3,4\}$. Then $\lambda(1,2)=\lambda(2,3)=\lambda(2,4)=1$, $\lambda(1,3)=\lambda(1,4)=0$, $\lambda(3,4)=2$.
\end{example}

The common block repeat degree $\lambda(x,y)$ has some properties as follows.

\begin{proposition}
\label{P:6}
(1) $\lambda(x,y)=\lambda(y,x)$; (2) $\lambda(x,y)\leq min(\partial(x),\partial(y))$.
\end{proposition}

\begin{proof}
\label{P:7}
It follows easily from Definition~\ref{D:Membership repeat degree} and Definition~\ref{D:Common block repeat degree}.
\end{proof}

It can be expressed by repeat degree that the set of the covering blocks including $x$ is equal to the set of the covering blocks including $x$ and $y$ simultaneously.

\begin{proposition}
\label{P:8}
Let $\mathbf{C}$ be a covering of a universe $U$. For any $x,y\in U$,
$\{K\in\mathbf{C}|x\in K\}=\{K\in\mathbf{C}|\{x,y\}\subseteq K\}\Leftrightarrow\partial(x)=\lambda(x,y)$.
\end{proposition}

\begin{proof}

$(\Rightarrow)$: It is straightforward.\\
$(\Leftarrow)$: It is clear that $\{K\in\mathbf{C}|\{x,y\}\subseteq K\}\subseteq\{K\in\mathbf{C}|x\in K\}$. If $\{K\in\mathbf{C}|\{x,y\}\subseteq K\}\neq\{K\in\mathbf{C}|x\in K\}$, therefore  $\{K\in\mathbf{C}|\{x,y\}\subseteq K\}$ is the proper subset of $\{K\in\mathbf{C}|x\in K\}$. Taking into account the finiteness of set $\{K\in\mathbf{C}|x\in K\}$, we have that $|\{K\in\mathbf{C}|\{x,y\}\subseteq K\}|<|\{K\in\mathbf{C}|x\in K\}|$, thus $\lambda(x,y)<\partial(x)$. This is a contradiction to that $\partial(x)=\lambda(x,y)$.

This completes the proof.
\end{proof}

Based on the concepts of membership repeat degree and common block repeat degree, we propose the concept of core block. Core block is a special kind of covering block and is closely related to the issue that under what condition neighborhoods induced by a covering are equal to the covering itself.

\begin{definition}(Core block)
\label{D:Core block}
Let $\mathbf{C}$ be a covering of a universe $U$. For any $x\in U$ and any $K\in\mathbf{C}$, $K$ is called the core block of $x$ if and only if $x\in K$ and for any $y\in K$, $\lambda(x,y)=\partial(x)$. The core block of $x$ is denoted as $\Gamma(x)$.
\end{definition}

For any element of $U$, say $x$, if it has a core block, are there some other different covering blocks which are the core blocks of $x$ as well? The following proposition answer this issue.

\begin{proposition}
\label{P:9}
Let $\mathbf{C}$ be a covering of a universe $U$. For any $x\in U$, if $K_{1},K_{2}\in\mathbf{C}$ are both the core block of $x$, then $K_{1}=K_{2}$.
\end{proposition}

\begin{proof}

By Definition~\ref{D:Core block}, we have that $x\in K_{1}$ and $x\in K_{2}$. For any $y\in K_{1}$, again, by Definition~\ref{D:Core block}, we have that $\partial(x)=\lambda(x,y)$. Then by Proposition~\ref{P:8}, we have that $\{K\in\mathbf{C}|x\in K\}=\{K\in\mathbf{C}|\{x,y\}\subseteq K\}$. As $x\in K_{2}$, thus $K_{2}\in\{K\in\mathbf{C}|x\in K\}$. So $K_{2}\in\{K\in\mathbf{C}|\{x,y\}\subseteq K\}$, then $\{x,y\}\subseteq K_{2}$, thus $y\in K_{2}$. Hence $K_{1}\subseteq K_{2}$. Similarly, $K_{2}\subseteq K_{1}$. Therefore $K_{1}=K_{2}$.

This completes the proof.
\end{proof}

This proposition indicates that the core block of any element of $U$ is unique.
It is possible that an element of a universe $U$ have no core block in a covering $\mathbf{C}$ of the universe $U$. To illustrate this, let us see an example.

\begin{example}
\label{E:12}
Let $U=\{1,2,3,4\}$, $\mathbf{C}=\{K_{1},K_{2},K_{3}\}$, where $K_{1}=\{1,2\}$, $K_{2}=\{1,2,3\}$, $K_{3}=\{3,4\}$. By the definition of core block, we see that $K_{1}$ is the core block of 1 as well as 2, namely $K_{1}=\Gamma(1)=\Gamma(2)$, and $K_{3}$ is the core block of 4, namely $K_{3}=\Gamma(4)$, but 3 have no core block.
\end{example}

By this example, we can also see that a block of a covering might be the core block of some different elements of the universe simultaneously. The following proposition give a necessary and sufficient condition for a covering block to be a core block.

\begin{proposition}
\label{P:10}
Let $\mathbf{C}$ be a covering of a universe $U$. For any $x\in U$, $K\in\mathbf{C}$ is the core block of $x$ if and only if $K$ is the intersection of all the blocks of $\mathbf{C}$ that include $x$.
\end{proposition}

\begin{proof}
Let $M=\{L\in\mathbf{C}|x\in L\}$. By $K\in\mathbf{C}$ and Proposition~\ref{P:8}, we have that\\
$K=\cap M\\
\Leftrightarrow(\cap M\subseteq K)\wedge(K\subseteq\cap M)\\
\Leftrightarrow(\cap M\subseteq K\wedge x\in K)\wedge(K\subseteq\cap M)\\
\Leftrightarrow(\cap M\subseteq K\wedge x\in K)\wedge(x\in K\wedge K\subseteq\cap M)\\
\Leftrightarrow(\cap M\subseteq K\wedge x\in K)\wedge((x\in K)\wedge\forall y((y\in K)\rightarrow(y\in\cap M)))\\
\Leftrightarrow(\cap M\subseteq K\wedge x\in K)\wedge((x\in K)\wedge\forall y((y\in K)\rightarrow\forall L((L\in\mathbf{C}\wedge x\in L)\rightarrow(y\in L))))\\
\Leftrightarrow(\cap M\subseteq K\wedge x\in K)\wedge((x\in K)\wedge\forall y((y\in K)\rightarrow\forall L((L\in\mathbf{C}\wedge x\in L)\rightarrow(\{x,y\}\subseteq L)))\\
\Leftrightarrow(\cap M\subseteq K\wedge x\in K)\wedge((x\in K)\wedge\forall y((y\in K)\rightarrow\forall L((L\in\mathbf{C}\wedge x\in L)\rightarrow(L\in\mathbf{C}\wedge\{x,y\}\subseteq L)))\\
\Leftrightarrow(\cap M\subseteq K\wedge x\in K)\wedge((x\in K)\wedge\forall y((y\in K)\rightarrow\forall L((L\in\mathbf{C}\wedge x\in L)\leftrightarrow(L\in\mathbf{C}\wedge\{x,y\}\subseteq L)))\\
\Leftrightarrow(\cap M\subseteq K\wedge x\in K)\wedge((x\in K)\wedge\forall y((y\in K)\rightarrow(\forall L(L\in\mathbf{C}\wedge x\in L)\leftrightarrow\forall L(L\in\mathbf{C}\wedge\{x,y\}\subseteq L)))\\
\Leftrightarrow(\cap M\subseteq K\wedge x\in K)\wedge((x\in K)\wedge\forall y((y\in K)\rightarrow(\{L\in\mathbf{C}|x\in L\}=\{L\in\mathbf{C}|\{x,y\}\subseteq L\}))\\
\Leftrightarrow(\cap M\subseteq K\wedge x\in K)\wedge((x\in K)\wedge\forall y((y\in K)\rightarrow(\partial(x)=\lambda(x,y)))\\
\Leftrightarrow(\cap M\subseteq K\wedge x\in K)\wedge(K=\Gamma(x))\\
\Leftrightarrow(x\in K)\wedge(K=\Gamma(x))\\
\Leftrightarrow(K=\Gamma(x))$.

This completes the proof.
\end{proof}

By Proposition~\ref{P:10}, we obtain the following corollary.

\begin{corollary}
\label{C:11}
Let $\mathbf{C}$ be a covering of a universe $U$. For any $x\in U$, if there exists the core block of $x$, then for any $K\in\mathbf{C}\wedge x\in K$, that $\Gamma(x)\subseteq K$ holds.
\end{corollary}

By Example~\ref{E:12}, we can also see that $K_{2}$ is not a core block of any element of $U$. The following proposition shows the characteristic of this kind of block in a covering.

\begin{proposition}
\label{P:13}
Let $\mathbf{C}$ be a covering of a universe $U$ and $K\in\mathbf{C}$. If $K$ is not a core block of any element of $U$, then $|K|>1$ and for any $x\in K$, $\partial(x)>1$.
\end{proposition}

\begin{proof}
Suppose that $|K|=1$, without loss of generality, suppose that $K=\{x\}$. Then $K$ is the intersection of all the blocks of $\mathbf{C}$ that include $x$. By Proposition~\ref{P:10}, we see that $K$ is the core block of element $x$. This is a contradiction to that $K$ is not a core block of any element of $U$.

It is clear that for any $y\in U$, $\partial(y)\geq1$. Suppose that there exists an element of $K$, say $x$, such that $\partial(x)=1$. Then for any $w\in K$, it follows that $\partial(x)=\lambda(x,w)=1$. Thus $K$ is the core block of element $x$. This is a contradiction to that $K$ is not a core block of any element of $U$.

This completes the proof.
\end{proof}

In a covering of a universe, it is possible that none of the whole blocks is a core block. To illustrate this, let us see an example.

\begin{example}
\label{E:14}
Let $U=\{1,2,3\}$, $\mathbf{C}=\{K_{1},K_{2},K_{3}\}$, where $K_{1}=\{1,2\}$, $K_{2}=\{2,3\}$, $K_{3}=\{1,3\}$. Then $K_{1}$, $K_{2}$ and $K_{3}$ are not core blocks of any element of $U$.
\end{example}

There might exist a block in a covering which is not a core block of any element of the universe, and even none of the whole blocks is a core block. When every element of the universe $U$ has its core block in the covering $\mathbf{C}$, is there a block in covering $\mathbf{C}$ which is not a core block of any element of the universe $U$? To solve this issue, we need to introduce the concept of reducible element. Furthermore, based on the concept of reducible element and the concept of invariable covering proposed in the following, we present a necessary and sufficient condition for neighborhoods induced by a covering to be equal to the covering itself.

\section{Condition for neighborhoods induced by a covering to
be equal to the covering itself}
\label{S:Condition for neighborhoods induced by a covering to
be equal to the covering itself}

To solve the issue of under what conditions two coverings generate
the same covering lower approximation or the same covering upper approximation, Zhu and Wang first proposed the
 the concept of reducible element in 2003. In order to obtain a necessary and sufficient condition under which neighborhoods induced by a covering are equal to the covering itself, we also need to use this concept.

\begin{definition}(Reducible element~\cite{ZhuWang03Reduction})
\label{D:15}
Let $\mathbf{C}$ be a covering of a universe $U$ and $K\in \mathbf{C}$. If $K$ is a union of some blocks in $\mathbf{C}-\{K\}$, we say $K$ is a reducible element of $\mathbf{C}$, otherwise $K$ is an irreducible element of $\mathbf{C}$.
\end{definition}

\begin{definition}~\cite{ZhuWang03Reduction}
\label{D:16}
Let $\mathbf{C}$ be a covering of $U$. If every element of $\mathbf{C}$ is an irreducible
element, we say $\mathbf{C}$ is irreducible; otherwise $\mathbf{C}$ is reducible.
\end{definition}

The following two proposition reveal the relationship between reducible element and core block.

\begin{proposition}
\label{P:17}
Reducible element of a covering is not core block.
\end{proposition}

\begin{proof}
Let $K$ be a reducible element of covering $\mathbf{C}$ of universe $U$. Then there exists a subset of $\mathbf{C}-\{K\}$, say $L$, such that $K=\cup L$. For any $P\in L$, it is clear that $P$ is a subset of $K$. Furthermore, we say that $P$ is a proper subset of $K$. Otherwise, we have that $P=K$. By $P\in L\subseteq\mathbf{C}-\{K\}$, we have that $K\in\mathbf{C}-\{K\}$. This is impossible.

Suppose $K$ be a core block of some element of $U$, say $x$. Then $x\in K$, thus there exists some $P\in L$, such that $x\in P$. By Corollary~\ref{C:11}, we have that $K\subseteq P$. This is a contradiction to that $P$ is a proper subset of $K$.

This completes the proof.
\end{proof}

The converse of this proposition is not true. From Example~\ref{E:14}, we can see that $K_{1}$, $K_{2}$ and $K_{3}$ are not core blocks of any element of $U$, but neither of them is reducible element. However, we have the following proposition which is related to this converse proposition.

\begin{proposition}
\label{P:18}
Let $\mathbf{C}$ be a covering of a universe $U$. Suppose that for any $x\in U$, there exists the core block of $x$ in covering $\mathbf{C}$ and that there exists $K\in\mathbf{C}$ which is not a core block of any element of $U$, then $K$ is a reducible element of $\mathbf{C}$.
\end{proposition}

\begin{proof}
By Proposition~\ref{P:13}, we have that $|K|>1$. Let $K=\{x_{1},x_{2},\cdots,x_{s}\}$, where $s\geq2$. By hypothesis, we see that for any $1\leq i\leq s$, $\Gamma(x_{i})\in\mathbf{C}$ and $\Gamma(x_{i})\neq K$. By Corollary~\ref{C:11}, we have that $\Gamma(x_{i})\subseteq K$, then $\cup_{i=1}^{s}\Gamma(x_{i})\subseteq K$. By $x_{i}\in\Gamma(x_{i})$, we have that $K\subseteq\cup_{i=1}^{s}\Gamma(x_{i})$. Thus $K=\cup_{i=1}^{s}\Gamma(x_{i})$.

This prove that $K$ is a reducible element of $\mathbf{C}$.
\end{proof}

The following example indicates that there exists the case described in Proposition~\ref{P:18}.

\begin{example}
\label{E:19}
Let $U=\{1,2,3\}$, $\mathbf{C}=\{K_{1},K_{2},K_{3},K_{4}\}$, where $K_{1}=\{1\}$, $K_{2}=\{2\}$, $K_{3}=\{3\}$, $K_{4}=\{1,2\}$. Then elements 1, 2 and 3 have their core blocks in covering $\mathbf{C}$, respectively. But $K_{4}$ is not a core block of any element of $U$. And $K_{4}=K_{1}\cup K_{2}$ is a reducible element of $\mathbf{C}$.
\end{example}

When all of the blocks of a covering $\mathbf{C}$ are core blocks, is there an element of the universe $U$ which has no core block in $\mathbf{C}$? The following example indicates that there exists this kind of case.

\begin{example}
\label{E:20}
Let $U=\{1,2,3\}$, $\mathbf{C}=\{K_{1},K_{2}\}$, where $K_{1}=\{1,2\}$, $K_{2}=\{2,3\}$. Then $K_{1}$ is the core block of 1, $K_{2}$ is the core block of 3. But element 2 has no core block in $\mathbf{C}$.
\end{example}

Based on the above conclusions, we propose the following concept.

\begin{definition}(Invariable covering)
\label{D:21}
Let $\mathbf{C}$ be a covering of a universe $U$. $\mathbf{C}$ is called an invariable covering if and only if $\mathbf{C}$ is irreducible and for any $x\in U$, there exists the core block of $x$.
\end{definition}

Invariable covering has the following property.

\begin{proposition}
\label{P:22}
Let $U$ be a universe. $\mathbf{C}$ is an invariable covering of $U$ if and only if for any $x\in U$, there exists the core block of $x$ and for any $K\in\mathbf{C}$, $K$ is the core block of some elements of $U$.
\end{proposition}

\begin{proof}
$(\Leftarrow)$: By the definition of invariable covering, we only need to prove that $\mathbf{C}$ is irreducible. We use an indirect proof. Suppose $\mathbf{C}$ be reducible. Then there exists at least one reducible element, say $K$, in covering $\mathbf{C}$. By Proposition~\ref{P:17}, we see that $K$ is not a core block of any element of $U$. This is a contradiction to the hypothesis.

$(\Rightarrow)$: Let $\mathbf{C}$ be an invariable covering of $U$. Then for any $x\in U$, there exists the core block of $x$. We only need to prove that for any $K\in\mathbf{C}$, $K$ is a core block of some elements of $U$. We use an indirect proof. Suppose that there exists some block of $\mathbf{C}$, say $K$, which is not a core block of any element of $U$. By Proposition~\ref{P:18}, we see that $K$ is a reducible element of $\mathbf{C}$. This is a contradiction to that $\mathbf{C}$ is irreducible.

This completes the proof.
\end{proof}

Proposition~\ref{P:22} can be considered as another definition of invariable covering.
Now, we present one of the main results in this paper. From this theorem, we will see that invariable covering is the only kind of covering which is equal to the neighborhoods induced by it.

\begin{theorem}
\label{T:23}
$Cov(\mathbf{C})=\mathbf{C}$ if and only if $\mathbf{C}$ is an invariable covering.
\end{theorem}

\begin{proof}
$(\Leftarrow)$: Let $\mathbf{C}$ be an invariable covering of $U$. For any $L\in\mathbf{C}$, by Proposition~\ref{P:22}, there exists some element of $U$, say $x$, such that $L=\Gamma(x)$. By Proposition~\ref{P:10}, we have that $\Gamma(x)=\cap\{K\in\mathbf{C}|x\in K\}=N(x)\in Cov(\mathbf{C})$. Then $L\in Cov(\mathbf{C})$. Thus $\mathbf{C}\subseteq Cov(\mathbf{C})$. Conversely, for any $M\in Cov(\mathbf{C})$, we see that there exists some element of $U$, say $y$, such that $M=N(y)=\cap\{K\in\mathbf{C}|y\in K\}$. Since there exists the core block of $y$ in $\mathbf{C}$, by Proposition~\ref{P:10}, we have that $\Gamma(y)=\cap\{K\in\mathbf{C}|y\in K\}$. Then $M=\Gamma(y)\in\mathbf{C}$. Thus $Cov(\mathbf{C})\subseteq\mathbf{C}$. Hence $Cov(\mathbf{C})=\mathbf{C}$.

$(\Rightarrow)$: Let $Cov(\mathbf{C})=\mathbf{C}$. Then $\mathbf{C}\subseteq Cov(\mathbf{C})$ and $Cov(\mathbf{C})\subseteq\mathbf{C}$. On the one hand, for any $L\in\mathbf{C}$, that $L\in Cov(\mathbf{C})$ holds. So there exists some element of $U$, say $x$, such that $L=N(x)=\cap\{K\in\mathbf{C}|x\in K\}$. By Proposition~\ref{P:10}, we have that $L=\Gamma(x)$. This indicates that all the blocks of $\mathbf{C}$ are core blocks. On the other hand, for any $y\in U$, that $N(y)\in Cov(\mathbf{C})$ holds. Thus $N(y)\in\mathbf{C}$. By Proposition~\ref{P:10} and $N(y)=\cap\{K\in\mathbf{C}|y\in K\}$, we have that $N(y)=\Gamma(y)$. Then $\Gamma(y)\in\mathbf{C}$. This indicates that every element of $U$ has its core block. By Proposition~\ref{P:22}, $\mathbf{C}$ is an invariable covering.

This completes the proof.
\end{proof}

\section{Condition for covering to be a neighborhoods}
\label{S:Condition for covering to be a neighborhoods}
Giving any covering $\mathbf{C}$ of a universe $U$, it is easy to calculate the neighborhoods out. But conversely, giving any covering $\mathbf{D}$ of the universe $U$,  it is not clear whether or not there exists a covering of the universe $U$, say $\mathbf{C}$, such that $D=Cov(\mathbf{C})$. Certainly, by the concept of $Cov(\mathbf{C})$ and some its properties, we know that if the amount of the blocks of covering $\mathbf{D}$ is more than the amount of the elements of universe $U$, or there exists some block of $\mathbf{D}$ which is a union of some other blocks of $\mathbf{D}$, namely, $\mathbf{D}$ is reducible, $\mathbf{D}$ must not be neighborhoods of any covering of universe $U$. But if a covering $\mathbf{D}$ does not belong to the cases as above mentioned, is it certainly a neighborhoods of some covering of universe $U$? To solve this issue, we firstly prove the following proposition about $Cov(\mathbf{C})$.

\begin{theorem}
\label{T:24}
For any covering $\mathbf{C}$ of universe $U$, it holds that $Cov(Cov(\mathbf{C}))=Cov(\mathbf{C})$.
\end{theorem}

\begin{proof}
We provide two proofs for this proposition.

The method one. By Theorem~\ref{T:23}, we only need to prove that $Cov(\mathbf{C})$ is an invariable covering. By Proposition~\ref{P:0}, we see that $\mathbf{C}$ is irreducible. For any $x\in U$, it is clear that $x\in N(x)$. And $\forall N(w)(N(w)\in Cov(\mathbf{C})\wedge x\in N(w)\rightarrow N(x)\subseteq N(w))$. This means that $N(x)$ is the intersection of all the blocks of $Cov(\mathbf{C})$ that include $x$. By Proposition~\ref{P:10}, we know that $N(x)$ is the core block of $x$. Thus $Cov(\mathbf{C})$ is an invariable covering. Hence $Cov(Cov(\mathbf{C}))=Cov(\mathbf{C})$.

The method two. Let $Cov(\mathbf{C})=\{N(x_{1}),N(x_{2}),\cdots,N(x_{m})\}$ and $Cov(Cov(\mathbf{C}))\\
=\{N^{\prime}(x_{1}),N^{\prime}(x_{2}),\cdots,N^{\prime}(x_{m})\}$. For any $1\leq i,j\leq m$, it is clear that $x_{i}\in N(x_{i})$. And if $x_{i}\in N(x_{j})$, we have that $N(x_{i})\subseteq N(x_{j})$. Thus $N^{\prime}(x_{i})=\cap\{N(x_{j})\in Cov(\mathbf{C})|x\in N(x_{j})\}=N(x_{i})$. Hence $Cov(Cov(\mathbf{C}))=Cov(\mathbf{C})$.

This completes the proof.
\end{proof}

This proposition is found and proved by ourselves independently. Afterward, we found that it is had been proved by Liu et al.~\cite{LiuSai09AComparison}. By this proposition, we have the following theorem.

\begin{theorem}
\label{T:25}
A covering $\mathbf{D}$ of universe $U$ is a neighborhoods of some covering of $U$ if and only if $Cov(\mathbf{D})=\mathbf{D}$.
\end{theorem}

\begin{proof}
$(\Leftarrow)$: If $Cov(\mathbf{D})=\mathbf{D}$, then $\mathbf{D}$ is the neighborhoods of covering $\mathbf{D}$.

$(\Rightarrow)$: Suppose $\mathbf{D}$ be a neighborhoods of some covering of $U$, say $\mathbf{C}$, i.e. $Cov(\mathbf{C})=\mathbf{D}$. By Theorem~\ref{T:24}, we have that $Cov(\mathbf{D})=Cov(Cov(\mathbf{C}))=Cov(\mathbf{C})=\mathbf{D}$.

This completes the proof.
\end{proof}

Of course, different coverings of universe $U$ can induce the same neighborhoods.

\section{Conclusions }
\label{S:Conclusions}
Neighborhood is an important concept in covering based rough sets. Through some concepts based on neighborhood and neighborhoods such as consistent function, we may find new connections between covering based rough sets and information systems. So it is necessary to study the properties of neighborhood and neighborhoods themselves. In this paper, we mainly studied on two issues induced by neighborhood and neighborhoods. The one is that under what condition neighborhoods induced by a covering is equal to the covering itself. The other one is that given an arbitrary covering, whether or not there exists another covering such that the neighborhoods induced by it is just the former covering. Through the study on the two fundamental issues, we have gained a deeper understanding of the relationship between neighborhoods and the covering which induce the neighborhoods. There are still many issues induced by neighborhood and neighborhoods to solve. We will continually focus on them in our following research.

\section*{Acknowledgments}
This work is supported in part by the National Natural Science Foundation of China under Grant No. 61170128, the Natural Science Foundation of Fujian Province, China, under Grant Nos. 2011J01374 and 2012J01294, and the Science and Technology Key Project of Fujian Province, China, under Grant No. 2012H0043.


\end{document}